\documentclass[sigconf,nonacm,balance=false]{acmart}
\usepackage{tabularx}
\usepackage{tikz}
\usetikzlibrary{arrows.meta,positioning,shapes.geometric,calc}
\usepackage{enumitem}
\usepackage{seqsplit}
\usepackage{placeins}
\usepackage{float}
\usepackage[safe]{pbalance}
\newcolumntype{Y}{>{\raggedright\arraybackslash}X}
\newcolumntype{L}[1]{>{\raggedright\arraybackslash}p{#1}}
\definecolor{reftitleblue}{RGB}{0, 79, 170}

\definecolor{revcolor}{RGB}{0, 0, 0}
\newcommand{\revadd}[1]{#1}
\newenvironment{revaddblock}{\par}{\par}

\theoremstyle{plain}
\newtheorem{theorem}{Theorem}
\newtheorem{proposition}{Proposition}

\theoremstyle{definition}
\newtheorem{definition}{Definition}
\newtheorem{assumption}{Assumption}
\theoremstyle{remark}
\newtheorem{remark}{Remark}

\AtBeginDocument{%
  \providecommand{\showDOI}[1]{\unskip}%
  \providecommand{\showURL}[1]{\unskip}%
  \renewcommand{\showDOI}[1]{\unskip}%
  \renewcommand{\showURL}[1]{\unskip}%
}

\begin{document}

\title{Context Is Not Authority: Structured Runtime Governance for Financial Market Agents}

\author{Rui Tang}
\affiliation{%
  \institution{OpenAsk}
  \country{}}
\email{vino@openask.me}

\author{Qiangqiang Liu}
\authornote{Corresponding authors: Qiangqiang Liu and Yichi Zhang.}
\affiliation{%
  \institution{Binance}
  \country{}}
\email{codi.l@binance.com}

\author{Yichi Zhang}
\authornotemark[1]
\affiliation{%
  \institution{New York University}
  \city{New York}
  \state{New York}
  \country{USA}}
\email{zhangyichi@stern.nyu.edu}

\author{Youwei Yang}
\affiliation{%
  \institution{Xiamen University}
  \city{Xiamen}
  \country{China}}
\email{yy783@cornell.edu}

\author{Xi Chen}
\affiliation{%
  \institution{New York University}
  \city{New York}
  \state{New York}
  \country{USA}}
\email{xc13@stern.nyu.edu}

\author{Chen Dong}
\affiliation{%
  \institution{Bank of Hebei}
  \country{China}}
\email{chendong7281@outlook.com}

\renewcommand{\shortauthors}{Tang et al.}

\begin{abstract}
Financial agents can turn correct context into an unauthorized effect: a customer-facing commitment, trade, or deployed policy. We present SAGE-Fin, a finance-specific authority-handoff contract that makes the proposed effect---not merely its text---the object of runtime control. SAGE-Fin compiles proposals into typed, adapter-bound candidates; preserves missing or stale institutional obligations as coverage debt; contracts permissible authority under current market, account, policy, and dialogue state; and requires an exact-artifact receipt whose nominal type matches the consuming response, execution, or policy adapter. Evidence and workflow progress therefore cannot substitute for effect authority, and prior authorization is rechecked after state change. Across an authored 616-case catalog, five deterministic specifications yield 3,080 outputs; a label-isolated harness obtains 616/616 binary reference--prototype parity, including 3/3 named response-gate fixtures, while 22 tests cover selected paths. These results establish executable conformance, not independent safety accuracy. Separately, SAGE-Fin's response gate processed real customer-facing production requests at a confidential digital-asset platform. A formal post-deployment review by an operational team organizationally independent of the implementation team reached a strongly positive conclusion on practical usefulness and workflow fit; end-user feedback was likewise strongly positive. Only the review's independence, stakeholder classes, assessed dimensions, and directional conclusion are disclosure-approved, so this is implementation-independent qualitative field corroboration rather than an aggregate effect estimate. Three distinct, de-identified predecessor failures, whose non-interception a second analyst independently confirmed (0/3), ground repeated-emission drift, stale account evidence, and missing escalation state without estimating prevalence or treatment effect.
\end{abstract}


\maketitle

\section{Introduction}

Financial AI systems are moving from passive analytics toward market agents that observe signals, retrieve context, reason over portfolio state, draft recommendations, propose risk-control changes, and prepare executable instructions. This shift extends broader FinTech and AI-in-finance trends in platform-mediated financial services, model-driven decision support, and automated financial workflows \cite{gomber2018fintech,cao2020ai,bajwa2022fintech}. In such systems, a failure can occur even when retrieved context is relevant, model output is fluent, and the underlying fact is correct. A momentum signal may be hardened into a trade, a stale quote may become a fee commitment, a diagnostic risk note may become personalized advice, or a dialogue-derived risk-control instruction may become an enforceable policy without deterministic validation. The common ingredient in these failures is natural language crossing an authority boundary: a fluent claim outruns the structured verification that would normally constrain it.

We study this failure mode as a runtime authority problem. Hallucination, factual error, and forecast error remain first-order risks in financial agents; SAGE-Fin addresses a complementary question that persists even after factuality and prediction checks succeed: whether a financial market-agent artifact has authority to be expressed, recommended, deployed, executed, or blocked under current state.

We call this the context-to-authority failure. Context includes market signals, portfolio state, retrieved documents, model analysis, previous dialogue, risk-control instructions, and operational metadata. Authority is narrower: the runtime permission to make a claim, issue a recommendation, reserve risk budget, deploy a policy, or send an instruction to a broker, exchange, or paper-trade adapter. Context can inform authority, but context is not authority.

SAGE-Fin addresses this gap through a finance-specific effect-boundary handoff contract. Its authorization unit joins the re-derived typed candidate, institutional witness/debt state, exact artifact, current state, and nominal receipt accepted by the consuming adapter. This is the paper's central research object: a content check may validate support, and a generic authorization substrate may encode predicates, but neither specifies this finance-domain object or the transfer of authority across response, execution, and policy seams. Missing or stale witnesses remain explicit debt; a dynamic controller contracts permissible authority; and only a valid \textsf{commit}, \texttt{execAuth}, or \texttt{deploy} receipt bound to the exact artifact can authorize the corresponding effect.

Tang et al. introduce CBEA+LCV to activate contract-bounded evidence and lexicographically validate structured language commitments before realization~\cite{tang2025recall}. SAGE-Fin governs an adjacent downstream boundary: whether a particular response, execution, or policy adapter may consume an exact authority-bearing artifact under mutable state. The mechanisms are complementary rather than interchangeable. An upstream validation result can discharge a registered obligation, but it is neither a downstream receipt nor permission to cause an effect.

The resulting contribution is a mechanism stack: a failure mode, a typed authority object, runtime soundness conditions, and executable mechanism evidence. Concretely, we make four contributions:
\begin{enumerate}
    \item We identify the \emph{context-to-authority} failure mode for financial agents: natural-language context can be relevant, fluent, and factually correct while still lacking runtime authority for the current claim, recommendation, policy, or execution adapter.
    \item We introduce SAGE-Fin, a finance-specific effect-boundary authority contract whose authorization unit joins a typed candidate, institutional witness/debt state, exact artifact, current state, and nominal adapter receipt. Under complete mediation, this composition yields receipt non-substitution, consumption-time revocation, and typed recovery.
    \item We formalize typed, non-substitutable, state-contingent receipts and prove trace-level authority non-amplification: non-consumption events cannot create governed effects; each effect requires a prior exact-artifact adapter receipt valid under current state; and progress or wrong-adapter receipts cannot substitute for effect authority.
    \item We evaluate the mechanism with 616 authored cases and 3,080 diagnostic outputs, 616/616 catalog-wide parity including 3/3 named prototype response-gate fixtures, and 22 tests; separately, we distinguish this public conformance evidence from later production deployment, a strongly positive formal review independent of implementers, strongly positive end-user feedback, and three distinct predecessor failures with independently confirmed 0/3 interception.
\end{enumerate}

The claim is intentionally scoped. SAGE-Fin does not decide whether a strategy is profitable. It does not establish causal alpha. It does not discover all validators autonomously. It does not certify regulatory compliance across jurisdictions. Its claim is that the runtime can determine whether a proposed financial action has authority under declared witnesses, validators, policies, receipts, and current state. This scope is closer to model-risk governance, auditability, and accountability mechanisms than to return prediction \cite{fedocc2011sr117,nist2023airmf,kroll2017accountable,raji2020closing}.

\section{Related Work}

\textbf{AI in finance and FinTech.} AI-in-finance research studies credit scoring, fraud detection, robo-advising, market prediction, risk surveillance, customer analytics, and financial automation \cite{cao2020ai,finra2020ai}, and FinTech research emphasizes that digital platforms alter interfaces, products, intermediation, and operational governance \cite{gomber2018fintech,bajwa2022fintech}. The recent generative-AI wave pushes language models closer to the runtime decision boundary that SAGE-Fin governs: finance-tuned LLMs such as BloombergGPT \cite{wu2023bloomberggpt}, multi-agent trading systems such as TradingAgents \cite{xiao2024tradingagents}, LLM-driven investment-management agents \cite{saha2025investment}, and the broader literature surveyed by Eisfeldt and Schubert \cite{eisfeldt2025genai} place fluent NL outputs directly upstream of order generation, portfolio actions, and customer-facing claims. SAGE-Fin is complementary to this work: it does not introduce a new alpha model or financial prediction task, but addresses the runtime boundary between financial context and authorized market-agent action.

\textbf{Model risk, governance, and auditability.} Financial model-risk management emphasizes validation, documentation, monitoring, policies, controls, and accountability for model use \cite{fedocc2011sr117}. Broader AI governance and accountability work asks how automated systems can be audited, reconstructed, and governed \cite{nist2023airmf,kroll2017accountable,raji2020closing}. Recent guidance specifically targets generative AI: the NIST Generative-AI Profile \cite{nist2024genai} and the EU AI Act \cite{eu2024aiact} impose transparency, oversight, incident-reporting, and audit obligations on high-risk AI deployed in financial services. SAGE-Fin extends these concerns to agentic runtime artifacts: the governed object is not only a trained model or an output log, but the authority-bearing transition from proposal to claim, policy, or execution.

\textbf{Financial conduct and reliance.} Retail financial communication rules and best-interest regimes constrain promissory, misleading, and personalized recommendation language \cite{finra2210,sec2019regbi}. Human-automation studies also show that users may over-rely on automated guidance when trust is miscalibrated \cite{parasuraman1997humans}. SAGE-Fin treats such reliance risks as authority-boundary problems: a diagnostic signal, retrieved context, or conversational statement should not automatically become a recommendation, remedy, or executable instruction.

\textbf{Structured language commitment control.} CBEA+LCV activates a bounded evidence set and validates candidate language commitments before surface realization \cite{tang2025recall}. That boundary is upstream of SAGE-Fin's exact-artifact handoff: linguistic support or a pre-realization validation result may satisfy a declared validator obligation, but cannot substitute for response, execution, or deployment authority. CBEA+LCV is adjacent prior work; the diagnostic language-control proxy evaluated below does not reproduce its selector or full pipeline.

Table~\ref{tab:positioning} makes this boundary explicit by separating what prior mechanisms already supply from SAGE-Fin's scoped contribution.

\begin{table*}[!t]
\caption{Positioning by mechanism boundary, not a claim of feature or expressiveness dominance.}
\label{tab:positioning}
\centering
\small
\setlength{\tabcolsep}{4pt}
\begin{tabularx}{\textwidth}{L{0.17\textwidth}L{0.34\textwidth}Y}
\toprule
Prior mechanism & What it already supplies & SAGE-Fin scoped contribution \\
\midrule
Complete mediation \cite{saltzer1975protection} & Recheck each mediated access. & Declare response, execution, and policy effects as finance-specific consumption points and bind each to a re-derived candidate. \\
ABAC \cite{hu2014abac} & Attribute-based decisions over subjects, objects, operations, and environments. & Compile a proposed financial effect into a typed candidate and materialize absent or stale institutional obligations as coverage debt. \\
UCON \cite{park2004ucon} & Mutable attributes and authorization during ongoing use. & Bind an exact artifact to a nominal adapter receipt, then revalidate financial state and feed reconciliation into later authority. \\
Macaroons \cite{birgisson2014macaroons} & Delegated credentials attenuated by contextual caveats. & Separate progress receipts from \textsf{commit}, \texttt{execAuth}, and \texttt{deploy}; one receipt type cannot stand in for another adapter's authority. \\
AgentSpec \cite{wang2026agentspec} & Domain-agnostic trigger--predicate--enforcement rules for LLM agents. & Supply the finance lifecycle and objects linking candidate, witness, debt, exact artifact, receipt, adapter outcome, and reconciliation. \\
CBEA+LCV \cite{tang2025recall} & Upstream evidence activation and pre-realization language-commitment validation. & Gate downstream exact-artifact consumption under mutable state; an upstream pass may satisfy an obligation but cannot itself authorize an effect. \\
\bottomrule
\end{tabularx}
\end{table*}

\textbf{Agentic systems, complete mediation, and usage control.} Recent agent work emphasizes tool use and interaction with external environments \cite{yao2023react}; AgentHarm and Agent Security Bench measure harmful and adversarial agent behavior \cite{andriushchenko2025agentharm,zhang2025asb}, while AgentSpec provides a generic runtime-constraint language for LLM agents \cite{wang2026agentspec}. SAGE-Fin also has classical security antecedents. Complete mediation requires each access to be checked \cite{saltzer1975protection}; ABAC evaluates attributes of subjects, objects, operations, and environments \cite{hu2014abac}; UCON extends authorization to mutable attributes and ongoing use \cite{park2004ucon}; and macaroons attenuate delegated authority through contextual caveats \cite{birgisson2014macaroons}. These mechanisms can encode individual predicates, but leave the finance-domain authorization object and cross-stage handoff to the application. SAGE-Fin contributes that compositional object and lifecycle: proposed effect, typed candidate, institutional obligations, exact artifact, current state, and nominal adapter receipt are joined at consumption. This is not a greater-expressiveness claim; it is the domain contract required to prevent evidence progress, workflow progress, or prior authorization from silently becoming a new financial effect.

\textbf{Crypto-asset platform risk.} Digital-asset platforms add fast-changing chain state, irreversible transfers, custody exposure, liquidity fragmentation, and cross-border retail access. International policy work highlights market integrity, investor protection, custody, operational resilience, and retail distribution risks in crypto-asset markets \cite{fsb2023crypto,iosco2023cda}, and the EU MiCA regime \cite{eu2023mica} now imposes binding authorization, conduct-of-business, and operational-resilience requirements on crypto-asset service providers across all twenty-seven member states. SAGE-Fin is designed for this environment: it makes platform state, policy state, evidence freshness, and execution authority explicit rather than assuming that retrieved context is sufficient.

\FloatBarrier

\section{Context-to-Authority Failures}
\label{sec:failures}

\subsection{A Motivating Pre-SAGE Production Failure}
\label{sec:motivating-case}

\begin{revaddblock}
Before formalizing the failure taxonomy, we ground the problem in one of three distinct, de-identified production failures observed in the predecessor response workflow before SAGE-Fin. This historical episode is separate from SAGE-Fin's later customer-facing deployment at the same confidential digital-asset setting. The episode is revealing because nothing on the surface changes: the same sourced, policy-compliant answer is acceptable once and unauthorized later. A user contacts a financial-services LLM customer service agent (e.g., on a brokerage, exchange, or wealth-management platform) and asks how to reset their identity verification. The agent's retrieval layer returns a standard operating procedure (SOP) document stating that reset is unavailable in restricted jurisdictions, and the agent emits a sourced response reflecting that SOP. Several turns later, after the user has discussed an unrelated account-recovery question, the user re-asks ``how to reset identity verification.'' The predecessor agent retrieves the same SOP and emits an effectively unchanged response; that repeated emission is the observed failure. Under a retrospective SAGE-Fin analysis, the earlier authority could not simply be reused, and the repeated response would require a separately authorized route. This is a motivating failure reconstruction, not a logged SAGE-Fin intervention or counterfactual treatment estimate.

The two responses are effectively unchanged. The retrieval and factual content remain correct. What changes between turns is the candidate type: at turn~$t_1$ the agent is responding to a first inquiry, where a sourced SOP response is authorized; at turn~$t_k$ the agent is responding after intervening dialogue and an unresolved repeat, where the earlier authority decision no longer applies. Retrieval and factual-correctness checks do not represent this distinction. The runtime must recompile the candidate and evaluate the currently authorized route.

The pattern can arise beyond customer service whenever unchanged context is consumed under a changed authority state. A stale fee quote retrieved correctly at $t_1$ becomes a misleading fee commitment when emitted at $t_k$ after fee schedules have been republished. A diagnostic risk note retrieved with correct citation becomes personalized advice when emitted into a conversation where the user has just disclosed their portfolio. A budget reservation receipt that legitimately holds risk for an intended order at $t_1$ becomes execution permission if the runtime fails to distinguish receipt types at $t_k$. A verbal risk-control instruction (``relax drawdown checks for institutional accounts after-hours'') is fluent and well-formed but underspecifies the drawdown window, the institutional-flag witness, the freshness cap, and the rollback path; nevertheless, an LLM agent may compile it into a deployable policy.

The common ingredient is that \emph{context} (retrieval results, prior turns, source documents, model confidence) is being treated as \emph{authority} (the runtime permission to emit a claim, to commit risk, to send an instruction to a venue, to deploy a policy). Sections~\ref{sec:failure-patterns}--\ref{sec:hmm} show how SAGE-Fin separates these two notions; Section~\ref{sec:eval-empirical} illustrates the separation through customer-facing scenarios.
\end{revaddblock}

\subsection{Recurring Failure Patterns}
\label{sec:failure-patterns}

Market agents combine multiple sources of information, most of which arrive in or are mediated by natural language: model-generated commentary, retrieved documents, dialogue with users, and verbal risk-control instructions. Hallucination and bad forecasts are obvious failures. The following six failures are harder because text may be factually correct, well sourced, and useful yet remain unauthorized for the action or claim it supports.

\textbf{Signal hardening.} A soft signal becomes a stronger claim than the evidence permits. \revadd{\emph{In trading agents}: a momentum signal becomes ``buy now'' or ``this will be profitable'' rather than a bounded diagnostic statement. \emph{In customer-facing agents}: a retrieved analyst comment becomes a personalized recommendation without the suitability checks that a recommendation triggers under FINRA Rule 2210~\cite{finra2210} or SEC Reg BI~\cite{sec2019regbi}. \emph{In risk-control workflows}: a research note suggesting elevated risk in a sector becomes a hard exposure cap before peer review.}

\textbf{Evidence staleness.} A quote, fee estimate, liquidity condition, chain status, or portfolio snapshot is used after its validity window. \revadd{\emph{Concretely}: a fee schedule retrieved at $t_1$ is emitted to a user at $t_k$ after a fee-promotion period has ended; a chain status retrieved during a brief mempool quiet period is used to commit a withdrawal time at $t_k$ when the mempool has since congested; a portfolio snapshot taken pre-rebalance is used to compute available margin post-rebalance.} The model retrieved the right object but used it outside its authority window.

\textbf{Coverage debt hiding.} Required witnesses are absent, but the generated response masks missing evidence with confident prose. \revadd{\emph{In execution adapters}: missing slippage assumptions or walk-forward backtest evidence is masked by ``backtest results suggest...'' phrasing. \emph{In customer-facing adapters}: missing user-jurisdiction or user-tier witnesses are masked by emitting a generic SOP whose preconditions the agent has not verified. \emph{In risk-control}: missing rollback-path or fixture-replay evidence is masked by ``policy will activate on threshold breach'' without specifying what happens if the threshold-breach witness is itself stale.}

\textbf{Receipt misuse.} A system records a budget reservation or progress receipt and later treats it as execution permission. \revadd{\emph{In trading}: an order management system reserves margin for an intended order and the reservation receipt is later cited as authorization to send the order to the venue, even though the reservation does not include the live-quote and risk-limit checks that final authorization requires. \emph{In customer support}: a prior clarification decision is later cited as authority to recommend a withdrawal path, even though no withdrawal-eligibility witness has been activated.} Entering a workflow is not final execution or response authority.

\textbf{Dialogue-to-policy deployment.} A human or model conversation produces risk-control intent, but the intent is ambiguous. \revadd{\emph{In risk-control workflows}: a verbal instruction ``relax drawdown checks for institutional accounts after-hours'' leaves open which drawdown window (rolling, daily, intraday), which witness substantiates the institutional flag, what fallback applies if the witness is stale, and what the rollback path is on a false trigger.} Different LLM agents may map the same phrase into different thresholds, scopes, or effects. \revadd{\emph{Especially relevant in crypto-asset venues}: where MiCA Title V requires conduct-of-business and operational resilience controls~\cite{eu2023mica}, ambiguous dialogue cannot satisfy the deterministic-policy requirement.} Dialogue can propose policy, but it cannot directly become deployed enforcement.

\revadd{\textbf{Authority drift across turns.} A candidate compiled with valid authority at turn $t_1$ is re-emitted at turn $t_k$ ($k > 1$) without recompilation, even though prior dialogue has changed the candidate type. \emph{In trading agents}: a buy recommendation valid pre-news is restated post-news without re-deriving the candidate. \emph{In customer-facing agents}: an SOP retrieval valid for a first inquiry is re-emitted to a user who has signaled the SOP is inapplicable, with no candidate-type re-derivation (Section~\ref{sec:motivating-case}). The same surface response carries different authority obligations in different conversational contexts; without per-turn recompilation through $\Gamma$ (Definition~\ref{def:candidate}), the runtime treats authority as cached when it should be re-derived.}

The six observed patterns map into the four threat classes used below. Signal hardening and coverage-debt hiding are two forms of adversary~A: fluent language requests more authority than declared evidence supports. Evidence staleness and authority drift across turns are forms of adversary~B: an earlier context or authorization is consumed after relevant state changes. Receipt misuse is adversary~C, and dialogue-to-policy deployment is adversary~D. This grouping preserves the finer failure taxonomy while giving each runtime defense a single threat-model locus.

These failures are especially relevant in financial markets because small wording and state changes can have material consequences. A statement about liquidity, fee, position size, leverage, withdrawal status, or trade suitability can shape user behavior even if no order is executed. Conversely, a system that touches execution must distinguish diagnostic context from authority to send an instruction to an adapter.

\revadd{\textbf{Applicability across financial domains.} The same interface can be instantiated in equity brokerage, derivatives quoting, wealth management, lending, and onboarding: each domain supplies its own candidate types, witness registry, freshness horizons, and forbidden effects. For example, a brokerage recommendation requires suitability and communication witnesses, while a derivatives quote requires a fresh option-chain snapshot and volatility bounds. These are design instantiations, not empirical cross-domain validation; the field evidence in Section~\ref{sec:eval-empirical} is limited to a customer-response gate.}

\section{Threat Model and Scope}
\label{sec:threat}

SAGE-Fin assumes a market-agent runtime that can observe proposals, candidate artifacts, declared witnesses, registry state, budget state, receipt issuance, adapter calls, and reconciliation outcomes. The threat model is dominated by natural-language inputs: the model or agent may generate fluent but overcommitted text, request an authority level stronger than its evidence supports, use stale context, omit required witnesses, or pass the wrong receipt type to a downstream adapter. We organize these threats as four NL-driven adversaries, each mapped to a SAGE-Fin defense from Sections~\ref{sec:formalization}--\ref{sec:hmm}.

\textbf{(A) Signal-hardening adversary.} Fluent NL converts a soft signal into a hard commitment (``momentum is strong'' becomes ``buy now''), requesting a cap stronger than the underlying evidence supports. \emph{Defense:} typed candidates make the requested effect explicit; coverage debt contracts authority; and the gate checks forbidden effects on the exact consumed artifact (Definitions~\ref{def:candidate}--\ref{def:vld}).

\textbf{(B) Stale-context adversary.} The adversary delays consumption of a receipt---or relies on context past its validity window---so that consumption state no longer matches issuance state. \emph{Defense:} the gate re-derives the candidate, evaluates bound state and expiry predicates, and compares the request with current controller authority; the causal filtration ensures any regime test uses only information available at $t$.

\textbf{(C) Receipt type confusion and cross-artifact reuse.} The adversary submits a progress or wrong-adapter receipt as authority, or reuses a valid receipt against another artifact, subject, or scope. \emph{Defense:} nominal receipt types, trusted issuance, and candidate, exact-artifact, subject, and scope bindings (Definition~\ref{def:vld}; Theorems~\ref{thm:trace-safety} and~\ref{thm:nonsub}). Same-artifact one-shot replay remains outside the current prototype.

\textbf{(D) Dialogue-policy injection adversary.} Ambiguous NL in dialogue is used to introduce risk-control intent that, if compiled naively, would lower thresholds or enlarge enforcement scope (e.g., ``relax drawdown checks for institutional accounts after-hours''). \emph{Defense:} dialogue is treated as untrusted proposal material (Section~\ref{sec:risk-policy}); enforcement requires deterministic policy IR + validation + fixtures + shadow + a deployment receipt with $H(\pi)$ binding; Theorem~\ref{thm:pol-sound} gives the soundness of the policy enforcement gate.

The system boundary is the runtime gate. SAGE-Fin enforces authority only for actions routed through receipt-checking adapters; external actions that bypass these adapters lie outside the trust boundary and appear in reconciliation as observations (with possible contraction via Proposition~\ref{prop:contraction}). The goal is not to prove that a trade is profitable, optimal, legally sufficient, or globally safe. The goal is to ensure that authority-increasing runtime actions over natural-language inputs are explicit, typed, state-contingent, receipt-bound, and reconstructable.

\section{SAGE-Fin Architecture}
\label{sec:architecture}

SAGE-Fin is a runtime authority control plane. It mediates authority-increasing transitions from market context to financial claims, recommendations, policy deployments, and executable instructions. Figure~\ref{fig:architecture} shows the runtime architecture: proposals enter at the top, the compiler $\Gamma$ produces a typed candidate, witness and validator registries determine coverage debt, and the authority controller emits two classes of receipts whose consumption is gated by type-restricted adapters; reconciliation feeds state back, and the audit ledger feeds promotion back into the registries. Each component is the runtime locus of one of the four NL adversaries enumerated in Section~\ref{sec:threat}: the compiler $\Gamma$ defends against signal-hardening (A) by forcing fluent NL through deterministic compilation; the authority controller and state-contingent receipt validity defend against stale-context (B); the type-restricted adapters realize receipt non-substitutability against type-confusion / replay (C); and the deterministic policy IR plus deployment-receipt pipeline guards dialogue-policy injection (D).

\begin{figure*}[t]
\centering
\begin{tikzpicture}[
  scale=0.85, transform shape,
  font=\small,
  >=Latex,
  node distance=5mm and 8mm,
  box/.style={draw, rounded corners=2pt, align=center,
              minimum height=7mm, inner sep=3pt, font=\small},
  input/.style={box, fill=gray!10},
  proc/.style={box, fill=blue!8},
  reg/.style={box, fill=yellow!18, dashed},
  audr/.style={box, fill=orange!15},
  authr/.style={box, fill=orange!16},
  gate/.style={box, fill=green!15, very thick},
  casegate/.style={box, fill=green!8, dashed, very thick},
  arr/.style={->, thick},
  fb/.style={->, dashed, thick}
]

\node[input] (sig) {market signal};
\node[input, right=of sig] (mod) {model analysis};
\node[input, right=of mod] (dlg) {dialogue intent};

\node[proc, below=8mm of mod] (cmp) {compiler $\Gamma$};

\node[proc, below=of cmp] (cand) {typed candidate $c_t$};

\node[reg, left=14mm of cand] (wr) {witness reg.\ $W_t$\\validator reg.\ $V_t$};

\node[proc, below=of cand] (cov) {coverage debt $D_t$};

\node[reg, right=14mm of cov] (hmm) {HMM regime witness\\$b_t,\ H(b_t)$};

\node[proc, below=of cov, minimum width=42mm] (ctrl) {authority controller\\[1pt]$A_{g,t}=G_g(\xi_t,c_t,b_t)$};

\node[audr, below left=9mm and 6mm of ctrl] (audR)
  {progress receipts\\\scriptsize evidence $\mid$ budget $\mid$ validation\\\scriptsize fixture $\mid$ shadow};
\node[authr, below right=9mm and 6mm of ctrl] (authR)
  {adapter receipts\\\scriptsize commit $\mid$ \texttt{execAuth} $\mid$ \texttt{deploy}};

\node[reg, below=of audR] (ledger) {audit \& promotion ledger};
\node[gate, below=of authR, xshift=-10mm] (gE) {execution adapter (gate)};
\node[gate, below=of authR, xshift=22mm] (gP) {policy adapter (gate)};
\node[casegate, below=of authR, xshift=56mm] (gR) {response adapter\\\scriptsize deployed path};

\node[proc, below=10mm of gP, xshift=8mm] (rec) {reconciliation};

\draw[arr] (sig) -- (cmp);
\draw[arr] (mod) -- (cmp);
\draw[arr] (dlg) -- (cmp);
\draw[arr] (cmp) -- (cand);
\draw[arr] (wr.east) -- (cand.west);
\draw[arr] (cand) -- (cov);
\draw[arr] (wr.south east) -- (cov.north west);
\draw[arr] (cov) -- (ctrl);
\draw[arr] (hmm.west) -- (ctrl.east);
\draw[arr] (ctrl.south west) -- (audR.north);
\draw[arr] (ctrl.south east) -- (authR.north);
\draw[arr] (audR) -- (ledger);
\draw[arr] (authR.south) -- (gE.north);
\draw[arr] (authR.south) -- (gP.north);
\draw[arr] (authR.south) -- (gR.north);
\draw[arr] (gE) -- (rec);
\draw[arr] (gP) -- (rec);
\draw[arr] (gR) -- (rec);

\coordinate (R1) at ([xshift=8mm]current bounding box.east);
\coordinate (R2) at (R1 |- hmm.east);
\draw[fb] (rec.east) -- (rec.east -| R1) -- (R1)
  -- node[right=2pt, font=\scriptsize, align=left]{$\xi_{t+1}$\\feedback} (R2)
  -- (hmm.east);

\coordinate (L1) at ([xshift=-8mm]current bounding box.west);
\coordinate (L2) at (L1 |- wr.west);
\draw[fb] (ledger.west) -- (ledger.west -| L1) -- (L1)
  -- node[left=2pt, font=\scriptsize]{promotion} (L2)
  -- (wr.west);

\end{tikzpicture}
\caption{SAGE-Fin runtime architecture. Inputs are proposals, not authority. The compiler $\Gamma$ produces a typed candidate $c_t$; registries expose coverage debt $D_t$; and the controller consumes current state $\xi_t$ and optional regime witness $b_t$. Progress receipts support audit and promotion. Scoped commitment, \texttt{execAuth}, and \texttt{deploy} receipts are accepted only by response, execution, and policy gates, respectively. Reconciliation updates future state, and the ledger feeds promotion back into the registries.}
\Description{Block diagram of the SAGE-Fin runtime architecture. Inputs feed a compiler, typed candidate, witness and validator registries, coverage debt, an optional regime witness, an authority controller, progress and adapter receipts, response, execution, and policy gates, reconciliation, and feedback loops.}
\label{fig:architecture}
\end{figure*}
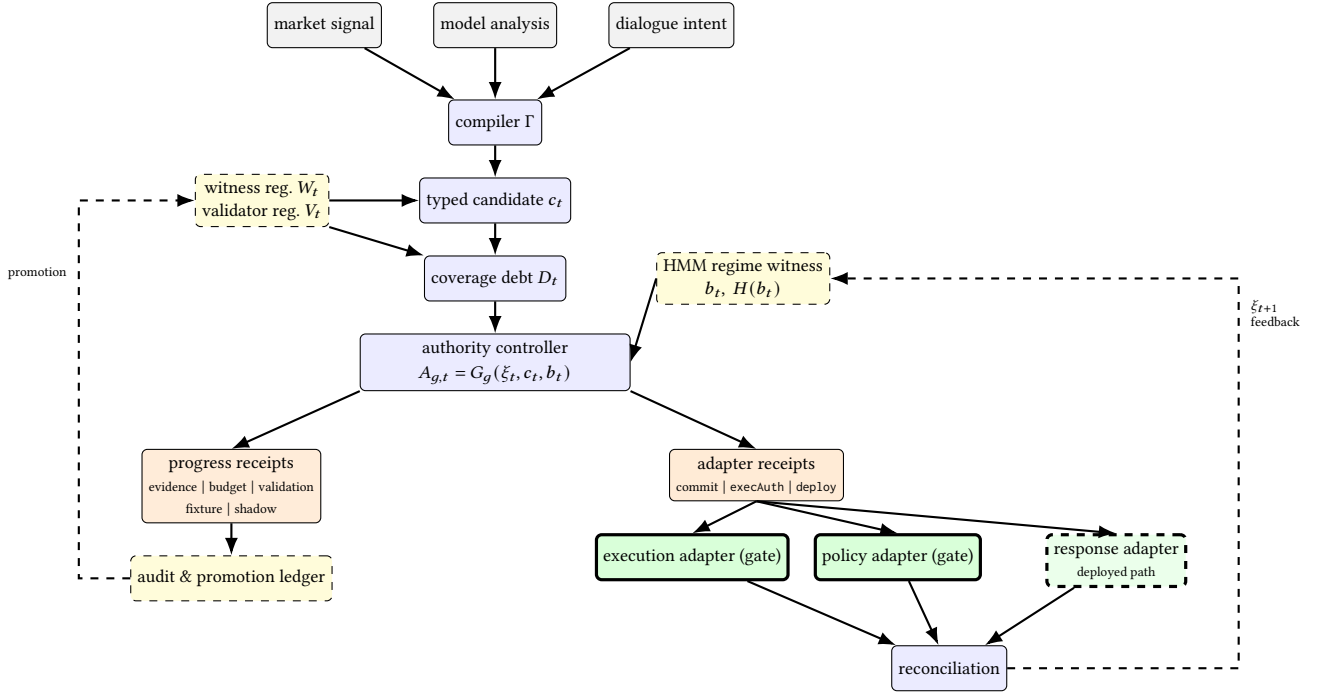

The financial execution lifecycle is:

\begin{quote}
\small
market signal / model analysis / dialogue intent $\rightarrow$ typed candidate $\rightarrow$ witness binding $\rightarrow$ coverage diagnostic $\rightarrow$ dynamic authority controller $\rightarrow$ quote $\rightarrow$ reserve $\rightarrow$ confirm reservation hold $\rightarrow$ final authorize $\rightarrow$ authorize execution $\rightarrow$ execute or block $\rightarrow$ reconcile $\rightarrow$ update future authority.
\end{quote}

The risk-control policy lifecycle is:

\begin{quote}
\small
dialogue $\rightarrow$ policy intent $\rightarrow$ policy candidate $\rightarrow$ deterministic policy IR $\rightarrow$ validation $\rightarrow$ fixtures $\rightarrow$ shadow mode $\rightarrow$ deployment receipt $\rightarrow$ runtime enforcement $\rightarrow$ rollback / reconciliation.
\end{quote}

\subsection{Typed Candidates}

A candidate is a structured artifact produced by the deterministic compiler $\Gamma$ from a market signal, model analysis, strategy output, or dialogue-derived intent --- most of which arrive in natural-language form. $\Gamma$ is therefore the trust-boundary transformation from natural language to runtime-checkable structure: any input that cannot be compiled is rejected at this stage rather than silently promoted. The candidate itself is not authority; it only makes requested authority explicit. Candidate fields include a candidate hash, source signal IDs, claim type, requested commitment, instrument or policy scope, required witnesses, forbidden claims, risk intent, and compiler trace.

The candidate type controls required evidence. A fee estimate may require current fee and quote witnesses. A live-trade candidate may require price freshness, risk controls, liquidity and slippage assumptions, walk-forward backtest evidence, out-of-sample evidence, and user authorization. A policy deployment candidate requires deterministic policy IR, passing validation, fixture replay, shadow evaluation, approval metadata, and registry installation.

\revadd{\textbf{Financial instantiations.} A \texttt{spot\_fee\_quote} candidate requires a fresh fee-schedule witness and a user-tier witness; a \texttt{margin\_authorization} candidate requires a portfolio-snapshot witness, a leverage-state witness, and a walk-forward backtest witness; a \texttt{withdrawal\_eligibility} candidate requires a chain-status witness, a user-jurisdiction witness, and a pending-withdrawal-balance witness; an \texttt{sop\_response} candidate (Section~\ref{sec:motivating-case}) requires a user-context witness and a prior-emission-absence witness; a \texttt{drawdown\_policy\_deployment} candidate requires deterministic-IR and validation witnesses, fixture-replay and shadow-mode witnesses, registry-promotion evidence, and an owner-approval witness. The identity-verification-reset inquiry in Section~\ref{sec:motivating-case} compiles into \texttt{sop\_response} on first emission and into the descriptive type \emph{post-clarification SOP response} on a subsequent emission in the same session, with different witness and forbidden-effect sets.}

\subsection{Witness and Validator Registries}

SAGE-Fin separates evidence from authority by requiring promoted witnesses and validators. A model may propose a witness or validator, but proposal status is not sufficient. Only promoted registry entries can satisfy authority requirements.

Example witness types include current price or quote snapshots, fee schedules, transaction-cost assumptions, portfolio exposure snapshots, drawdown and leverage state, liquidity and slippage estimates, walk-forward backtest records, out-of-sample results, risk-control objects, policy registry state, user authorization, and latent regime witnesses.

Validators check freshness, scope, forbidden language, risk controls, policy conflicts, receipt type, and adapter compatibility. The registry makes incompleteness explicit: missing validators do not become silent assumptions; they become coverage debt or unsupported scope.

\revadd{\textbf{Witnesses by adapter type.} Execution-adapter witnesses include the live order book, instrument-level circuit-breaker status, tick-size and lot-size constraints, and the daily-loss-limit register. Policy-adapter witnesses include the deterministic policy IR with bound hash $H(\pi)$, the fixture-replay outcome, the shadow-mode decision-divergence summary, and the owner-approval signature. Customer-facing adapter witnesses include user-jurisdiction status (subject to sanctions and MiCA territorial scope~\cite{eu2023mica}), user-tier status (which controls fee schedules and limit envelopes), prior-emission registers (which detect repeated SOP emissions in the same session), and escalation-trigger witnesses (which fire when the SOP terminal state is reached without resolution). Validators include freshness checks scoped per witness type (chain status: seconds; fee schedules: minutes; portfolio snapshots: a single snapshot per consumption; SOPs: per-session per-candidate-type), scope checks (the consumed instruction must lie within the witness's instrument and venue scope), and policy-conflict checks (the candidate must not collide with an active emergency-halt policy).}

\subsection{Coverage Debt}

Coverage debt is the set of required witnesses or validators that are missing, stale, unpromoted, or invalid for a candidate. It is runtime state, not a footnote. Coverage debt can downgrade commitment level, block execution authorization, or require clarification, repair, escalation, or human approval. SAGE-Fin does not claim complete validator coverage. It prevents missing coverage from being hidden inside generated prose.

\revadd{\textbf{Financial coverage-debt scenarios.} A live-trade candidate that has a fresh price witness but no slippage-assumption witness has \emph{partial} coverage debt: SAGE-Fin downgrades to \texttt{paper\_trade} mode and surfaces the missing assumption in the receipt's downgrade-reason field rather than emitting a live-trade authorization. A fee-quote candidate whose user-tier witness is older than the tier-recompute interval has \emph{stale-witness} coverage debt: SAGE-Fin downgrades the response from a specific fee number to a generic account path. A repeated SOP candidate with no authorized escalation route has \emph{blocking} coverage debt: the gate blocks re-emission, while human handoff requires an independently declared fail-safe. A policy-deployment candidate whose fixture-replay witness has not been produced has \emph{blocking} coverage debt: the deployment receipt is not issued.}

\subsection{Dynamic Authority Controller}

Authority is not a static label; it is state-contingent and adapter-specific. The controller consumes market state, portfolio state, coverage debt, witness freshness, policy registry state, open reservations, receipts, reconciliation history, incident state, and optional regime witnesses. In the reference execution model, allowed modes form the chain
\[
\begin{aligned}
\bot=\text{blocked} &\prec \text{diagnostic} \prec \text{educational}\\
&\prec \text{paper\_trade} \prec \text{live\_small}=\top .
\end{aligned}
\]
The execution adapter combines this mode with bounded size, slippage, freshness, and expiry coordinates. Response and policy adapters use their own cap orders (Definition~\ref{def:lattice}); the prototype does not expose unrestricted live trading.

\subsection{Receipts and Runtime Gates}

Receipts are scoped artifacts. They bind a candidate, exact artifact, subject, evidence and policy state, cap, and validity window. Receipt types are non-substitutable. Evidence, budget, validation, fixture-test, and shadow-mode receipts support audit or promotion but grant no effect authority.

Adapter receipts include a scoped commitment for response emission, \texttt{execAuth} for execution, and \texttt{deploy} for policy enforcement. Each adapter rejects missing, stale, wrong-type, wrong-scope, hash-mismatched, or state-incompatible receipts. A failed gate does not itself authorize a fallback: generic response, contracted execution, or human handoff must be independently declared and satisfy its own predicates.

\revadd{\textbf{Non-substitutability in financial workflows.} On a trading desk, a margin-reservation receipt (\texttt{bud}) issued at $t_1$ does not authorize an order send at $t_k$, because the reservation does not certify that the live-quote and risk-limit witnesses required by the execution adapter are still valid. On a risk-management workstation, a fixture-replay receipt (\texttt{fxt}) confirming that a draft drawdown policy fires correctly on historical stress traces does not authorize the policy adapter to enforce the policy on a live event, because enforcement requires the deployment receipt (\texttt{deploy}) which additionally binds the deterministic policy IR hash $H(\pi)$ and the owner-approval witness. On a customer-service console, a clarification-question receipt does not authorize a withdrawal recommendation, because the recommendation requires a user-eligibility witness and a fresh fee-schedule witness that the clarification step did not capture. In each case, ``progress through the workflow'' is recorded as auditable evidence, but final authorization requires its own typed receipt with its own witness set.}

\subsection{Reconciliation}
\label{sec:reconciliation}

After execution or enforcement, SAGE-Fin reconciles the observed outcome against the authorized artifact. Reconciliation records whether the action obeyed the receipt, underfilled, overfilled, mismatched scope, exceeded slippage bounds, lacked acknowledgement, or triggered incident review. Formally, let $o_t\in\mathcal{O}$ denote the reconciliation outcome and $\xi_t$ the runtime state at consumption. The feedback map
\[
\xi_{t+1}=\mathrm{Reconcile}(\xi_t,o_t)
\]
updates the incident state $I_{t+1}\supseteq I_t$, the portfolio drawdown component of $p_{t+1}$, and the freshness state $\Phi_{t+1}$. The map is monotone with respect to incidents: once an incident enters $I_t$, future authority on the affected scope is bounded above by a contracted cap until the incident is cleared (Proposition~\ref{prop:contraction}).

\revadd{\textbf{Reconciliation in trading vs.\ customer-facing workflows.} In a trading workflow, $o_t$ records fill price, fill quantity, slippage versus the authorized envelope, broker acknowledgement timing, and any post-trade incident flag. A slippage breach makes the incident coordinate less permissive and contracts later authority on the affected scope. In a customer-facing workflow, $o_t$ records whether the SOP path reached its terminal node, whether the user re-asked the same intent within $K$ turns, and any post-emission compliance flag. A repeated inquiry reclassifies the next candidate from \texttt{sop\_response} to a post-clarification SOP response; candidate-hash invalidation then requires a fresh receipt.}

\section{Formalization}
\label{sec:formalization}

This section formalizes SAGE-Fin's runtime objects (candidates, registries, coverage debt, authority states, receipts), the dynamic authority controller, and the soundness properties of the runtime gates. Notation is summarized in Table~\ref{tab:notation}.

\begin{table}[!tbp]
\caption{Notation summary.}
\label{tab:notation}
\small
\begin{tabularx}{\columnwidth}{lY}
\toprule
Symbol & Meaning \\
\midrule
$P_t,\Gamma$ & proposal context; deterministic compiler \\
$c_t^\star,\theta_t$ & re-derived typed candidate; nominal type \\
$\phi_t,s_t,a_t^{req}$ & content or action; scope; requested cap \\
$W_t^{req},V_t^{req},E_t,x_t$ & witness and validator predicates; forbidden effects; exact artifact \\
$W_t,V_t$ & active witness / validator registry \\
$D_t,D_t^{blk}$ & coverage debt; blocking coverage debt \\
$\xi_t$ & consumption-time runtime state \\
$\Phi_t,\Pi_t,\mathcal B_t,\mathcal R_t,\mathcal H_t,I_t,\Delta_t$ & freshness, policy, reservation, receipt, history, incident, dialogue/drift state \\
$b_t,H(b_t)$ & regime posterior; normalized regime entropy \\
$\mathcal{F}_t$ & filtration generated by observations up to $t$ \\
$g,\mathcal A_g,\preceq_g,\wedge_g$ & adapter and its finite authority semilattice \\
$A_{g,t}$ & current adapter-specific authority cap \\
$R,\mathsf{Accept}(g)$ & receipt; receipt types accepted by adapter $g$ \\
$U_t$ & upstream evidence-activation or language-validation result \\
$\mathrm{Reconcile}$ & state feedback map after outcome $o_t$ \\
\bottomrule
\end{tabularx}
\end{table}

\subsection{Candidates, registries, and coverage debt}

\begin{definition}[Typed authority request]\label{def:candidate}
At adapter consumption, a deterministic compiler re-derives
\[
c_t^\star=\Gamma(P_t,\xi_t)
=(\theta_t,\phi_t,s_t,a_t^{req},W_t^{req},V_t^{req},E_t),
\]
where $\theta_t$ is a nominal candidate type, $\phi_t$ the proposed content or action, $s_t$ its scope, $a_t^{req}$ the adapter-specific requested cap, $W_t^{req}$ and $V_t^{req}$ the required witness and validator predicates, and $E_t$ the forbidden effects. Let $x_t$ be the exact serialized artifact presented to adapter $g$. Re-derivation makes a dialogue or state change authority-relevant even when surface text is unchanged. Compilation creates a request, not authority.
\end{definition}

\begin{remark}[Surface--authority separation]\label{rem:surface-authority}
Let $q_t$ denote the authority signature presented at consumption: candidate type, adapter, nominal receipt type, scope, validity predicate, receipt-issued cap, and current adapter cap. Define $x_t\equiv_{\mathrm{text}}x_{t'}$ when two realized artifacts have the same surface content, and $q_t\equiv_{\mathrm{auth}}q_{t'}$ when those authority coordinates agree. Whenever $\Gamma$ or the gate depends on a state coordinate that may change across turns, surface equality does not imply authority equality:
\[
x_t\equiv_{\mathrm{text}}x_{t'}
\;\not\Rightarrow\;
q_t\equiv_{\mathrm{auth}}q_{t'}.
\]
The same text can therefore require a different candidate type, witness set, receipt, or fallback after dialogue, freshness, portfolio, reservation, incident, or policy state changes.
\end{remark}

\begin{definition}[Runtime state and filtration]\label{def:state}
The runtime state available at consumption is
\[
\xi_t=(m_t,p_t,\Phi_t,\Pi_t,\mathcal B_t,\mathcal R_t,\mathcal H_t,I_t,\Delta_t),
\]
with market, portfolio, freshness, policy, reservation, receipt, history, incident, and dialogue or drift coordinates. Let $\mathcal{F}_t$ denote the filtration generated by observations through $t$; gates use only $\mathcal{F}_t$-measurable state. If a latent regime witness is present (Section~\ref{sec:hmm}), $b_t$ denotes its filtered posterior.
\end{definition}

\begin{definition}[Coverage debt]\label{def:debt}
For active witnesses $W_t$ and validators $V_t$, $D_t(c_t^\star)$ contains each predicate in $W_t^{req}\cup V_t^{req}$ whose corresponding witness or validator is absent, stale, unpromoted, provenance-invalid, or scope-incompatible. Blocking debt $D_t^{blk}$ contains predicates that the target adapter requires; other debt may contract the available cap.
\end{definition}

\begin{proposition}[Monotonicity of authority in coverage debt]\label{prop:cov-monotone}
For fixed adapter $g$ and candidate $c$, if $D\subseteq D'$ and the registered coverage term is contractive in debt, then $A_g^{cov}(D')\preceq_g A_g^{cov}(D)$.
\end{proposition}
\begin{proof}[Proof sketch]
This is the declared contract of the coverage term: adding an unmet obligation cannot increase its cap. Meet composition preserves the order.
\end{proof}

\subsection{Authority lattice and dynamic controller}

\begin{definition}[Adapter authority semilattice]\label{def:lattice}
Each adapter $g$ has a finite meet-semilattice $(\mathcal A_g,\preceq_g,\wedge_g,\top_g)$. Its ordered caps are meaningful only for that adapter: response commitments may order emission classes, execution caps may combine mode, size, slippage, and expiry, and policy caps may order enforcement scope. Candidate types and adapters are nominal and are not ordered against one another.
\end{definition}

\begin{definition}[Dynamic authority controller]\label{def:controller}
For adapter $g$, the controller returns
\[
A_{g,t}=A_g^{base}(c_t^\star)\wedge_g A_g^{cov}(D_t)
\wedge_g A_g^{state}(\xi_t)\wedge_g A_g^{reg}(b_t).
\]
The state term incorporates registered portfolio, reservation, freshness, policy, incident, and dialogue predicates. If the candidate and adapter do not register a regime-witness obligation, $A_g^{reg}=\top_g$. A required regime witness that is missing or stale instead enters coverage debt or contracts the state term according to its registered fail-closed policy. Each non-base term is contractive: it can preserve or reduce nominal authority but cannot increase it.
\end{definition}

\begin{remark}[Conservative composition]
The meet form prevents a strong signal from compensating for missing witnesses, failed validators, portfolio stress, policy conflict, or regime stress. A receipt valid at issuance can therefore become invalid before consumption.
\end{remark}

\subsection{Receipts and runtime gates}

\begin{definition}[Receipt]\label{def:receipt}
A receipt is the tuple
\[
R=(\kappa_R,i_R,u_R,h_R^c,h_R^x,s_R,a_R,h_R^\Pi,t_i,t_e,\chi_R),
\]
where $\kappa_R$ is the nominal type, $i_R$ a trusted issuer, $u_R$ the user, account, or policy subject, $h_R^c$ and $h_R^x$ the candidate and exact-artifact hashes, $s_R$ the scope, $a_R\in\mathcal A_g$ the issued cap, $h_R^\Pi$ the policy hash, $t_i,t_e$ the issue and expiry times, and $\chi_R$ the state predicates that must hold at consumption. Receipt types are
\[
\begin{aligned}
\mathsf{Types}=\{&\mathsf{evd},\mathsf{bud},\mathsf{val},\mathsf{fxt},\mathsf{shd},\\
&\mathsf{commit},\mathsf{execAuth},\mathsf{deploy}\}.
\end{aligned}
\]
\end{definition}

\begin{definition}[Consumption-time gate]\label{def:vld}
Let
\[
\begin{aligned}
\mathsf{Accept}(\mathsf{response})&=\{\mathsf{commit}\},\\
\mathsf{Accept}(\mathsf{exec})&=\{\mathsf{execAuth}\},\\
\mathsf{Accept}(\mathsf{policy})&=\{\mathsf{deploy}\}.
\end{aligned}
\]
A receipt is valid for adapter $g$ only if $\kappa_R\in\mathsf{Accept}(g)$; $i_R$ belongs to the trusted issuer set and $\mathsf{Issued}_{i_R}(R)$ holds; $t_i\leq t\leq t_e$; $h_R^c=H(c_t^\star)$ and $h_R^x=H(x_t)$; the subject, scope, and policy bindings match; and $\chi_R(\xi_t)$ holds. The adapter accepts iff
\[
\begin{aligned}
\mathsf{Gate}_g(x_t,c_t^\star,R,\xi_t)=\mathbf 1[&
\mathsf{vld}_g(x_t,c_t^\star,R,\xi_t,t)\\
&\wedge\ \mathsf{cap}_g(x_t)\preceq_g a_t^{req}
\preceq_g a_R\wedge a_t^{req}\preceq_g A_{g,t}\\
&\wedge\ D_t^{blk}(c_t^\star)=\varnothing
\wedge \mathsf{Effects}_g(x_t)\cap E_t=\varnothing].
\end{aligned}
\]
Otherwise the adapter rejects $x_t$. Any downgrade, clarification, or handoff must be independently declared, compiled, and gated.
\end{definition}

\begin{assumption}[Mediated trusted consumption]\label{ass:gates}
Every governed effect occurs only as an adapter-consumption event after gate acceptance, and $\mathsf{Issued}_{i_R}(R)$ holds only after the corresponding trusted issuance event. The compiler, registries, witness providers, clocks, state feeds, hash bindings, and internal receipt issuer are trusted.
\end{assumption}

\begin{theorem}[Trace-level authority non-amplification]\label{thm:trace-safety}
Let $\tau=(e_1,\ldots,e_n)$ be any finite sequence of proposal, witness/validator, receipt-issuance, state-update, and adapter-consumption events. Under Assumption~\ref{ass:gates}, every governed effect event $e_t$ has an earlier trusted issuance event for an accepted-type receipt whose candidate, exact artifact, subject, scope, cap, policy, time, and state bindings hold at $t$, with no blocking debt. Extending a trace only with non-consumption events cannot add a governed effect; a receipt whose bound predicate becomes false cannot authorize a later effect.
\end{theorem}
\begin{proof}[Proof sketch]
Induct on $n$. The empty trace is immediate. A final non-consumption event adds no governed effect by complete mediation. If $e_n$ is a governed effect, mediation requires $\mathsf{Gate}_g=1$; Definition~\ref{def:vld} supplies the earlier trusted issuance, accepted nominal type, exact bindings, current validity, empty blocking debt, permitted effect, and requested/issued/current cap bounds. A progress receipt fails the type check, and a false bound predicate makes $\mathsf{vld}_g=0$. Applying the induction hypothesis to $(e_1,\ldots,e_{n-1})$ proves the claim.
\end{proof}

\begin{theorem}[Receipt non-substitutability]\label{thm:nonsub}
Progress and wrong-adapter receipts cannot substitute for \textsf{commit}, \textnormal{\texttt{execAuth}}, or \textnormal{\texttt{deploy}}, because their nominal types are outside the corresponding acceptance set.
\end{theorem}

\begin{proposition}[Upstream-validation non-escalation]\label{prop:upstream-nonescalation}
Let $U_t$ be an upstream evidence-activation or language-validation result whose registered effect is limited to adding a witness or satisfying a validator predicate. For every adapter $g$,
\[
U_t=\mathsf{pass}\ \wedge\
\bigl(R=\varnothing\ \vee\ \kappa_R\notin\mathsf{Accept}(g)\bigr)
\quad\Longrightarrow\quad
\mathsf{Gate}_g=0.
\]
An upstream pass may reduce coverage debt and allow a trusted issuer to consider a later adapter receipt. It cannot itself change a receipt's nominal type, supply trusted issuance, bind the exact artifact, or satisfy consumption-time state predicates.
\end{proposition}
\begin{proof}[Proof sketch]
Definition~\ref{def:vld} requires a present receipt whose nominal type lies in $\mathsf{Accept}(g)$, in addition to exact-artifact, scope, time, policy, cap, and current-state checks. Failure of the receipt condition is sufficient for rejection regardless of $U_t$ or the resulting coverage state.
\end{proof}

\begin{theorem}[Soundness of the execution gate]\label{thm:exec-sound}
If the execution adapter accepts instruction $x_t$, then Theorem~\ref{thm:trace-safety} holds with $\kappa_R=\mathsf{execAuth}$, $\mathsf{cap}_{\mathsf{exec}}(x_t)\preceq_{\mathsf{exec}}a_R$, and $\mathsf{cap}_{\mathsf{exec}}(x_t)\preceq_{\mathsf{exec}}A_{\mathsf{exec},t}$.
\end{theorem}

\begin{theorem}[Soundness of the policy enforcement gate]\label{thm:pol-sound}
If the policy adapter enforces policy artifact $x_t$, then Theorem~\ref{thm:trace-safety} holds with $\kappa_R=\mathsf{deploy}$ and $h_R^x=H(x_t)$.
\end{theorem}

\begin{remark}[Liveness is bounded by coverage]
The gate results are soundness, not liveness. SAGE-Fin does not guarantee that any candidate receives authority; persistent blocking debt or stress may leave no admissible effect. Liveness depends on registry coverage and validator availability outside the runtime guarantee.
\end{remark}

\begin{proposition}[Conditional authority contraction]\label{prop:contraction}
For fixed adapter $g$, candidate $c$, and regime witness $b$, let $\xi'\sqsubseteq_g\xi$ mean that every registered state coordinate in $\xi'$ is no more permissive than its counterpart in $\xi$. If each controller term is order-preserving under this adverse-state preorder, then
\[
A_g(c,\xi',b)\preceq_g A_g(c,\xi,b).
\]
Candidate reclassification is handled separately by hash invalidation and receipt re-issuance.
\end{proposition}
\begin{proof}[Proof sketch]
Meet monotonicity preserves the componentwise order of the registered controller terms.
\end{proof}

Proposition~\ref{prop:contraction} concerns controller-cap monotonicity. Receipt revocation does not depend on its premise: failure of any predicate bound in $\chi_R$ already makes $\mathsf{vld}_g=0$ by Definition~\ref{def:vld}; Theorem~\ref{thm:revoke} gives the execution-regime instantiation.

Candidate and exact-artifact bindings reject cross-artifact reuse. Preventing repeated consumption of the same receipt for the same artifact requires an atomic spent-receipt ledger and is not provided by the current prototype.

\section{Adapter-Specific Instantiations}
\label{sec:instantiations}

The preceding contract is adapter-generic. We next instantiate it at two financial boundaries: an optional latent-regime witness that can contract execution authority, and a dialogue-to-policy path that requires deterministic IR before enforcement. These examples share the same candidate--debt--receipt--gate lifecycle; they are not separate authorization models.

\subsection{Execution: Latent Regime Witness}
\label{sec:hmm}

SAGE-Fin may incorporate a latent market-regime model as a witness provider. We instantiate this with a Hidden Markov Model (HMM)~\cite{rabiner1989tutorial}. The HMM is not an alpha model; it is an auditable source of dynamic authority state whose output can contract, shorten, or revoke authority but cannot by itself justify a trade.

\subsubsection{Regime witness}

Let the latent regime be $z_t\in\mathcal{Z}=\{N,V,S,I\}$ (normal, volatile, stressed, illiquid). The observation vector is
\[
y_t=(r_t,\sigma_t,\delta_t,\nu_t,\lambda_t,\eta_t,d_t)\in\mathbb{R}^7,
\]
with $r_t$ return, $\sigma_t$ realized volatility, $\delta_t$ bid-ask spread, $\nu_t$ volume, $\lambda_t$ a liquidity feature, $\eta_t$ a slippage estimate, and $d_t$ a drawdown-related feature. The HMM has initial distribution $\pi\in\Delta_{\mathcal{Z}}$, transition matrix $T_{ij}=P(z_t=j\mid z_{t-1}=i)$, and Gaussian emissions with diagonal covariance:
\[
y_t\mid z_t=j\;\sim\;\mathcal{N}(\mu_j,\,\mathrm{diag}(\Sigma_j)),\qquad j\in\mathcal{Z}.
\]
The emission family is interface-local: substitutes (Gaussian mixtures, Student-$t$, heavy-tailed) do not change the authority interface in Definition~\ref{def:hmm-witness}.

\begin{definition}[HMM regime witness]\label{def:hmm-witness}
The regime witness at time $t$ is the triple
\[
\mathcal{W}^{hmm}_t=(b_t,H(b_t),\mathrm{fresh}^{hmm}_t),
\]
with filtered posterior $b_t(j)=P(z_t=j\mid y_{1:t})$, normalized entropy $H(b_t)=-\frac{1}{\log|\mathcal{Z}|}\sum_j b_t(j)\log b_t(j)$, and freshness indicator $\mathrm{fresh}^{hmm}_t=\mathbf{1}[t-\mathrm{ts}(y_t)\leq h^{hmm}]$, where $\mathrm{ts}(y_t)$ is the wall-clock time of the most recent observation consumed by the filter.
\end{definition}

The provider computes $b_t$ with the standard forward recursion and fits $\Theta=(\pi,T,\{\mu_j,\Sigma_j\}_{j\in\mathcal{Z}})$ by Baum--Welch on a rolling window $\mathcal W_t\subset(-\infty,t)$. Declared state-order constraints and overlap-window alignment keep labels stable across refits. These estimation choices are local to the witness provider: the authority interface consumes only $(b_t,H(b_t),\mathrm{fresh}^{hmm}_t)$ and never a smoothed posterior.

\subsubsection{Causal filtration and walk-forward soundness}

Let $\mathcal{F}_t=\sigma\bigl(y_{1:t},\hat\Theta_t\bigr)$ be the filtration generated by observations through time $t$ and by the rolling estimator available at $t$.

\begin{proposition}[Causal filtering]\label{prop:causal}
$b_t$ and $H(b_t)$ are $\mathcal{F}_t$-measurable. Smoothed posteriors $P(z_t\mid y_{1:T})$ for any $T>t$ are not used by the authority controller in fixture evaluation. Hence each evaluated authority decision at time $t$ is $\mathcal{F}_t$-measurable.
\end{proposition}
\begin{proof}[Proof sketch]
The forward recursion writes $b_t$ as a deterministic function of $b_{t-1}$, $y_t$, and $\hat\Theta_t$, all $\mathcal{F}_t$-measurable; iterated composition stays in $\mathcal{F}_t$. The rolling-window constraint $\max\mathcal{W}_t<t$ ensures $\hat\Theta_t$ uses no future observations.
\end{proof}

This is the formal walk-forward statement: regime witnesses cannot import future observations into past authority computations, and rolling re-fits cannot peek beyond $t$.

\subsubsection{Regime-to-authority map}

The HMM is an optional contractive witness for the execution adapter, not an expected-return input. Let $\mathcal Q_{\mathsf{exec}}$ be the finite execution-size buckets and $\mathcal M_{\mathsf{exec}}=\{\bot,\mathrm{diag},\mathrm{paper},\mathrm{live\_small}\}$ the execution-mode chain. For an executable candidate the size cap factorizes
\[
q_{\mathsf{exec},t}^{max}=\operatorname{bucket}_{\mathcal Q_{\mathsf{exec}}}\!\left(q_0(c_t^\star)\alpha(p_t)\rho(\mathcal B_t)\beta(D_t)\gamma(b_t)\right),
\]
with contraction multipliers $\alpha(p_t),\rho(\mathcal B_t),\beta(D_t),\gamma(b_t)\in[0,1]$ for portfolio stress, reserved budget or exposure on the consumed scope, coverage debt, and regime risk. The bucket map returns the greatest declared size bucket not exceeding the raw cap. Let $r_t=\mathbf 1[b_t(S)>\theta_S\lor b_t(I)>\theta_I\lor H(b_t)>\theta_H]$. Matching the executable prototype and reference evaluator, we instantiate $\gamma$ as a fail-closed threshold map:
\[
\gamma(b_t)=
\begin{cases}
0, & r_t=1,\\
1, & r_t=0,
\end{cases}
\]
and the allowed mode is contracted symmetrically:
\[
M_t\preceq
\begin{cases}
\mathrm{diag}, & \mathrm{fresh}^{hmm}_t=0\ \text{or}\ r_t=1,\\
\mathrm{live\_small}, & \text{otherwise}.
\end{cases}
\]

\begin{proposition}[Monotone size and mode contraction]\label{prop:hmm-monotone}
Let
\[
q^{max}(p,\mathcal B,D,b)=
\operatorname{bucket}_{\mathcal Q_{\mathsf{exec}}}\!\left(q_0(c)\alpha(p)\rho(\mathcal B)\beta(D)\gamma(b)\right),
\]
where $\alpha$ is non-increasing in the declared portfolio-risk order, $\rho$ is non-increasing in reserved budget or exposure on the consumed scope, $\beta$ is non-increasing in coverage-debt inclusion and severity, and $\gamma$ is the thresholded regime contraction above. If $p,\mathcal B,D,b$ are weakly no safer than $p',\mathcal B',D',b'$ under these orders, then
\[
q^{max}(p,\mathcal B,D,b)\leq q^{max}(p',\mathcal B',D',b').
\]
Let $M(b,f)$ denote the mode cap above, with $f=1$ fresh and $f=0$ stale. If $f\leq f'$, $b(S)\geq b'(S)$, $b(I)\geq b'(I)$, and $H(b)\geq H(b')$, then
\[
M(b,f)\preceq M(b',f')
\]
in $\mathcal M_{\mathsf{exec}}$. Thus worse portfolio risk, greater reservations, greater coverage debt, higher stress or illiquidity posterior, or higher entropy cannot increase either the size cap or the allowed mode; freshness loss cannot increase the mode.
\end{proposition}
\begin{proof}[Proof sketch]
Each multiplier is non-increasing in its corresponding risk order, and the product of nonnegative non-increasing multipliers is non-increasing in the product order. The bucket map is monotone, so bucketization preserves the inequality. The mode map is piecewise constant, non-decreasing in freshness $f$ where $1$ denotes fresh, and non-increasing in stress, illiquidity, and entropy. Direct case analysis on the thresholds gives the mode inequality.
\end{proof}

\subsubsection{State-contingent receipt revocation}

An execution receipt issued at $t_0$ may bind thresholds $(\theta_S^R,\theta_I^R,H_R)$ inside its consumption-time state predicate $\chi_R$, defining the authorized regime set
\[
\mathcal{B}_R=\{b\in\Delta_{\mathcal{Z}}:b(S)\leq\theta_S^R,\;b(I)\leq\theta_I^R,\;H(b)\leq H_R\}.
\]
The gate then yields the following revocation result.

\begin{theorem}[State-contingent revocation]\label{thm:revoke}
Let $R^{\mathsf{execAuth}}$ bind thresholds $(\theta_S^R,\theta_I^R,H_R)$, expiry $t_e$, and a fresh regime witness. If at consumption any of
\[
b_t(S)\leq\theta_S^R,\quad b_t(I)\leq\theta_I^R,\quad H(b_t)\leq H_R,\quad t\leq t_e,\quad \mathrm{fresh}^{hmm}_t=1
\]
fails, the execution gate rejects the receipt.
\end{theorem}
\begin{proof}[Proof sketch]
Each condition is bound into $\chi_R$ or the receipt time interval. Definition~\ref{def:vld} requires both at consumption; failure makes $\mathsf{vld}_{\mathsf{exec}}=0$. Proposition~\ref{prop:causal} ensures the regime test uses only $\mathcal F_t$.
\end{proof}

\begin{remark}[Reauthorization, not repair]
A revoked receipt is not repaired by attaching new evidence or a new entropy bound. The candidate must be re-evaluated by the controller under the current state and, if feasible, a new $R^{\mathsf{execAuth}}$ issued. This preserves the soundness theorem (Theorem~\ref{thm:exec-sound}) under regime transitions.
\end{remark}

\paragraph{Evaluation use.}

The HMM component is evaluated by gate behavior on stored regime states, not market prediction or regime-estimation accuracy. Tests check whether stressed, illiquid, or extreme-entropy states contract execution caps, reject stale witnesses, and revoke incompatible receipts. A falsifying result is an execution authorization that remains valid after a bound condition in Theorem~\ref{thm:revoke} fails.

\subsection{Policy: Dialogue to Deterministic Risk-Control IR}
\label{sec:risk-policy}

Dialogue can reveal risk-control intent, but dialogue is not deployable policy. A statement such as ``block strategies when drawdown exceeds five percent'' leaves open the scope, event type, witness freshness, fallback, effect, owner, expiry, and rollback semantics. Different models may interpret the same instruction differently.

SAGE-Fin treats dialogue-derived policy content as untrusted proposal material. A minimal financial risk-control policy IR contains scope, trigger, deterministic conditions, required witnesses, effect, fallback, owner, validity window, and rollback pointer. The runtime adapter enforces only policies with valid deployment receipts. Candidate, validation, fixture, and shadow-mode receipts support audit and promotion decisions, but they do not authorize enforcement.

\textbf{Worked example: dialogue $\to$ IR $\to$ fixture $\to$ shadow $\to$ deploy.} Consider the utterance ``block strategies when drawdown exceeds five percent.'' The compiler $\Gamma$ emits a \emph{policy candidate} with $\mathrm{scope}=\texttt{all\_strategies}$, $\mathrm{trigger}=\texttt{drawdown}$, $\mathrm{threshold}=\texttt{5\%}$, $\mathrm{effect}=\texttt{block}$, and leaves several fields unresolved: which drawdown window, which witness substantiates drawdown, what fallback applies if the witness is stale, and what rollback path handles a false trigger? Each unresolved field enters coverage debt $D_t$. Resolution proceeds through deterministic IR validation, fixture replay, a declared non-enforcing shadow evaluation, registry promotion, and owner-approved deployment. Validation, fixture, and shadow receipts record progress but remain outside $\mathsf{Accept}(\mathsf{policy})$; only a $\mathsf{deploy}$ receipt bound to the exact policy artifact and its promotion evidence can authorize enforcement. Fluent dialogue alone never produces deployment authority.

This is one financial instantiation of the same authority principle: dialogue is not policy authority, signal is not trade authority, and model output is not runtime authority.

\section{Prototype and Evaluation}
\label{sec:eval}

Evidence comes from four separately interpreted sources: a reference catalog; prototype tests and catalog-wide cross-implementation parity; three distinct production failures from the predecessor response workflow with second-analyst confirmation; and later bounded operation of the SAGE-Fin response gate with a formal implementation-independent operational review. The catalog tests decision conformance and clause reach; parity tests mapped implementation behavior on that same authored catalog; the historical episodes and second-analyst confirmation of 0/3 predecessor interception ground the failure modes; and deployment, a strongly positive formal review independent of implementers, and strongly positive end-user feedback provide qualitative field corroboration of practical usefulness and workflow fit, while no disclosure-approved aggregate effectiveness estimate is reported. Section~\ref{sec:replay} specifies a prospective aggregate replay, which contributes no result here.

\subsection{Prototype Slice}
\label{sec:eval-prototype}

We implement a local SAGE-Fin slice with a market-claim JSON schema, a dynamic budget allocator over regime, drawdown, exposure, symbol caps, and freshness, a gate with approved, qualified, escalated, and rejected outcomes, a budget workflow (quote, reserve, confirm hold, release, authorize), evidence and budget ledgers, non-execution reservation-hold records, workflow metrics, and auditable local records. Illustrative cases---an AAPL momentum candidate rejected at final authorization despite a held reservation and an MSFT small-trade fixture approved only under compatible budget state---are summarized in Tables~\ref{tab:workflow} and~\ref{tab:budget-sweep}.

Twenty-two automated tests exercise selected wrong-type receipt, candidate and exact-instruction binding, issue-to-consumption, budget, reconciliation, regime, policy-revocation, active-incident, entropy, and deploy-only paths; all pass. The execution adapter recomputes the current candidate hash rather than trusting the issuance candidate, and policy deployment fails closed without passing validation, fixture replay, shadow evaluation, approval, and registry-promotion evidence. These tests establish behavior of the tested prototype paths. Cross-implementation binary parity is evaluated separately over the authored catalog below.

The prototype workflow demonstrates non-substitutability in isolation. A live AAPL momentum candidate receives a reservation and a non-execution reservation-hold record, but final authorization rejects it because of stale evidence, profit language, missing risk controls, and missing walk-forward evidence. A live MSFT small-trade fixture is approved under normal budget with a bounded maximum position. Under stressed, drawdown-exceeded, exposure-exhausted, or tight-symbol-cap states, the live-trade fixture is rejected or downgraded.

\begin{table}[!tbp]
\caption{Budget-workflow and final-authorization outcomes.}
\label{tab:workflow}
\begingroup
\scriptsize
\setlength{\tabcolsep}{2.6pt}
\begin{tabularx}{\columnwidth}{L{0.29\columnwidth}L{0.22\columnwidth}L{0.20\columnwidth}Y}
\toprule
Candidate & Request & Budget state & Final status \\
\midrule
BTC risk & Diagnostic & Held & Approved \\
ETH fee & Fee estimate & Held & Qualified \\
AAPL momentum & Live trade & Held & Rejected \\
NVDA fit & Recommendation & Held & Escalated \\
MSFT small & Live trade & Held & Approved \\
\bottomrule
\end{tabularx}
\endgroup
\end{table}

\begin{table}[!tbp]
\caption{Prototype dynamic-budget sweep.}
\label{tab:budget-sweep}
\small
\begin{tabularx}{\columnwidth}{L{0.31\columnwidth}ccY}
\toprule
Scenario & Live status & Max pos. & Main reason \\
\midrule
Normal budget & Approved & 0.50\% & -- \\
Tight symbol cap & Rejected & 0.10\% & Position exceeds budget \\
Exposure exhausted & Rejected & 0.00\% & Risk budget exhausted \\
Drawdown exceeded & Rejected & 0.00\% & Drawdown exceeded \\
Stressed market & Rejected & 0.00\% & Stressed regime \\
\bottomrule
\end{tabularx}
\end{table}

\FloatBarrier

\subsection{Executable Conformance Study}
\label{sec:eval-mechanism}

\textbf{What the catalog measures.} A separate reference evaluator runs deterministic diagnostic specifications and ablations over structured inputs, not language-model outputs. Each case provides a candidate, issuance and consumption states, witnesses, receipts, and authored expected decisions. Agreement is therefore specification conformance. Approvals within the full-policy rejection set locate clauses absent from each specification; they are not independently labeled errors, production leakage, or model performance.

The catalog contains $N{=}616$ authored cases: 28 curated fixtures spanning market-commitment compilation (12), receipt transition under state change (8), and policy-IR deployment (8); a 12-cell execution/policy receipt-type submatrix; and 576 generated fixtures crossing freshness, regime, regime entropy, drawdown, position size, and language. Three named cases traverse the prototype response-gate interface; none replays production traffic or tests deployed-code parity. Five deterministic specifications run on every case, yielding $616 \times 5 = 3{,}080$ case--specification outputs. The 29 full-policy approvals exercise permissive branches. Production episodes remain outside this catalog.

Table~\ref{tab:baselines} identifies the five reference specifications evaluated on the authored catalog; the first three are diagnostic proxies and the final two are SAGE-Fin ablations. No LLM baseline or aggregate production-traffic comparator outputs are reported; the three case-conditioned predecessor misses are analyzed separately in Section~\ref{sec:eval-empirical}.

\begin{table}[!tbp]
\caption{Reference specifications evaluated on the authored catalog.}
\label{tab:baselines}
\scriptsize
\begin{tabularx}{\columnwidth}{L{0.34\columnwidth}Y}
\toprule
Specification & Purpose \\
\midrule
Surface proxy & Tests whether natural-language instructions alone constrain unauthorized authority transitions. \\
Static checklist & Tests non-dynamic rule enforcement without receipt validity or reconciliation. \\
Language-control proxy & Checks prohibited wording and evidence presence; it does not reproduce the CBEA selector or full LCV pipeline~\cite{tang2025recall}. \\
Issuance-only SAGE-Fin & Tests typed artifacts and receipts without state-contingent consumption checks. \\
Full SAGE-Fin evaluator & Tests dynamic authority, non-substitutable receipts, adapter gates, and reconciliation. \\
\bottomrule
\end{tabularx}
\end{table}

\begin{table}[!tbp]
\caption{Designed decision contrasts on $N{=}616$ authored cases. Counts show approvals within the 587 cases that the full policy defines as non-approvable; they are mechanism contrasts, not accuracy estimates.}
\label{tab:conformance}
\begingroup
\scriptsize
\setlength{\tabcolsep}{3pt}
\begin{tabularx}{\columnwidth}{Ycc}
\toprule
Specification & \shortstack{Approvals\\/ 587} & Rate \\
\midrule
Surface proxy & 296/587 & 50.4\% \\
Static checklist & 294/587 & 50.1\% \\
Language-control proxy & 296/587 & 50.4\% \\
Issuance-only SAGE-Fin & 61/587 & 10.4\% \\
Full SAGE-Fin evaluator & 0/587 & 0.0\% \\
\bottomrule
\end{tabularx}
\endgroup
\end{table}

All 3,080 reference-evaluator binary decisions match their authored expected decisions. Because rules and labels are authored together, this is conformance rather than independent safety accuracy. On the 587 cases the full policy defines as non-approvable, the surface proxy, static checklist, language-control proxy, and issuance-only SAGE-Fin specifications approve 296, 294, 296, and 61 cases. The full row is zero by construction against its own rejection set. Of all 616 cases, 526 fail the issuance-only specification and 90 enter full evaluation; 87 contain explicit consumption-state objects, while three exercise diagnostic or receipt-type acceptance without one. Of the 90, 29 remain approved and 61 expose full-versus-issuance distinctions. Table~\ref{tab:seam-isolation} isolates four designed mechanism contrasts within this authored catalog.

\begin{table}[!htbp]
\caption{Designed seam-isolation contrasts in the authored reference catalog. These are mechanism checks, not independently labeled accuracy estimates.}
\label{tab:seam-isolation}
\centering
\scriptsize
\setlength{\tabcolsep}{2.5pt}
\begin{tabularx}{\columnwidth}{L{0.27\columnwidth}L{0.29\columnwidth}Y}
\toprule
Isolated seam & Matched contrast & Reference outcome \\
\midrule
Consumption-time revalidation & Eight named receipts valid at issuance, each followed by a bound adverse-state change & 8/8 issuance approvals become 8/8 consumption rejections \\
Receipt nominality & Twelve execution/policy receipt--adapter pairings & 2/2 correct pairings accept; 10/10 wrong pairings reject \\
Consumption filtering & Ninety cases admitted by the issuance-only specification & 29 retain approval; 61 are rejected under full evaluation \\
Typed response effect (Def.~\ref{def:candidate}) & Three named cases traverse the prototype response-gate interface & 3/3 parity: one approval and two rejections \\
\bottomrule
\end{tabularx}
\end{table}

\textbf{Reference--prototype parity.} A separate harness freshly runs the issuance-only and full reference evaluators and maps all 616 fixture inputs through the prototype's market, execution, and policy interfaces. Neither execution path reads authored labels; frozen full-policy oracles are joined afterward as a separate consistency check. Binary decisions match on 616/616 cases with zero mapping exceptions. Three named fixtures traverse the public prototype response-gate interface and match 3/3: one fresh diagnostic is approved, while a stale fee claim and personalized-advice boundary are rejected. This is public response-boundary conformance, not production traffic replay or deployed-code parity. Stratified by the fresh issuance-only reference run, 526 cases terminate at issuance and 90 reach full evaluation; the latter contain 87 explicit consumption-state objects and split into 29 approvals and 61 rejections. The complete run includes all 28 named fixtures and the 12-cell receipt matrix. Parity-harness tests freeze the catalog size, verify label isolation, exercise representative adapter paths and the candidate-bound \texttt{execAuth} positive control, and require zero mapping exceptions. Because the adapter mapping was developed against this same authored catalog, 616/616 establishes catalog-wide binary implementation parity, not independent correctness, safety accuracy, reason-code identity, general semantic equivalence, or production effectiveness.

These current results support mechanism conformance and implementation parity. The independently labeled replay needed to estimate prevalence or effect is specified after the field evidence in Section~\ref{sec:replay}.

\FloatBarrier

\subsection{Confidential Deployment and Predecessor Episodes}
\label{sec:eval-empirical}

SAGE-Fin's response-side gate later ran in the customer-facing serving path of a confidential digital-asset platform and processed real production requests. A formal post-deployment review by an operational team organizationally independent of the implementation team reached a strongly positive conclusion on practical usefulness and workflow fit; end-user feedback was likewise strongly positive. Only the review's independence, stakeholder classes, assessed dimensions, and directional conclusion are disclosure-approved. The formal review therefore supplies implementation-independent qualitative corroboration of deployment value; neither source yields an aggregate or comparative effect estimate. This deployment evidence is separate from the three episodes below. Each episode was a distinct production failure in the predecessor response workflow before SAGE-Fin; none is presented as an online SAGE-Fin decision or blocked event. The execution and policy adapters evaluated in this paper remain prototype and reference-model results. Institutional privacy and de-identification constraints also preclude disclosure of raw dialogue, customer attributes, exact wording, timestamps, thresholds, and venue-sensitive rules.

The three episodes---not multiple stages of one interaction---are distinct predecessor-workflow failures concerning different authority-state transitions. They are the only episodes approved for de-identified academic disclosure; their number reflects the institutional disclosure boundary, not a research sampling design. They are not synthetic fixtures: only their disclosed representation is de-identified, paraphrased, and reduced to control-relevant authority-state fields. A second analyst independently confirmed that the predecessor workflow intercepted none of the three episodes (0/3). For each, we retain the observed predecessor transition and separately derive its retrospective implication under the declared SAGE-Fin response policy.

\textbf{Scenario A: cached authority across dialogue.} In one predecessor episode, a sourced SOP answer is reused after intervening account-recovery dialogue. The observed failure is that effectively unchanged text retains its earlier authority despite prior emission and an unresolved inquiry. Under the declared SAGE-Fin policy, those state changes require candidate recompilation and current dialogue witnesses before another response.

\textbf{Scenario B: stale account-scoped evidence.} In a separate predecessor episode, a generic schedule is used for an account-specific fee response although the effective fee depends on a non-default tier and the tier witness is absent or stale. The observed failure is an account-specific commitment unsupported by current account evidence. Under the declared SAGE-Fin policy, the response contracts to clarification or a generic account path until a fresh account-bound tier witness is available.

\textbf{Scenario C: terminal workflow with missing routing state.} In a third predecessor episode, a repeated restricted-jurisdiction inquiry has been reclassified into a terminal workflow, but the next transition proceeds without the required escalation-route predicate. The observed failure is progression without current routing authority. Under the declared SAGE-Fin policy, the implication is to block re-emission and use an independently declared human-review fail-safe, not to infer escalation authority from missing state.

The episodes are distinct: Scenario~A concerns stale candidate identity, Scenario~B stale account evidence, and Scenario~C missing routing state after correct reclassification. Scenarios~A and C are not stages of the same interaction. All three preserve plausible surface content while lacking authority for the predecessor action.

\begin{table}[H]
\caption{Three distinct predecessor-workflow production failures; a second analyst confirmed predecessor interception was 0/3.}
\label{tab:cs-scenarios}
\centering
\small
\setlength{\tabcolsep}{2pt}
\begin{tabularx}{\columnwidth}{L{0.17\columnwidth}L{0.27\columnwidth}L{0.24\columnwidth}Y}
\toprule
Case & Unintercepted predecessor failure & Required input & Retrospective policy implication \\
\midrule
Repeated SOP & Reused earlier response authority & Fresh type; prior-emission state & Recompile; block repeat \\
Fee quote & Account commitment from stale/generic evidence & Fresh account-tier witness & Clarify or use generic path \\
Terminal SOP & Progressed without route state & Escalation-route predicate & Block; use human fail-safe \\
\bottomrule
\end{tabularx}
\end{table}

\paragraph{Case abstraction and de-identification.}
The episodes are structured, non-representative abstractions. We retain only the predecessor transition, requested commitment, changed authority input, and retrospective policy implication; time is relational, account state is bucketed as present, stale, or missing, and dialogue is reduced to prior-emission, terminal-state, and escalation-witness fields. Source interactions are paraphrased, and raw messages, identifiers, KYC records, session tokens, exact timestamps and thresholds, internal policy text, platform-specific control names, and production paths are excluded. This preserves the authority-state distinction without disclosing customer- or institution-sensitive content.

\paragraph{Evidence boundary and replay unit.}
The live deployment documents bounded in-situ use; the formal review organizationally independent of implementers and end-user feedback strongly corroborate practical usefulness and workflow fit. This field evidence is not linked to the three historical episodes. The second-analyst 0/3 confirmation establishes case-level predecessor non-interception for the disclosure-approved failures. Because institutional approval limits disclosure to three episodes, this evidence does not estimate a production-wide rate. Their SAGE-Fin actions remain analytic counterfactual mappings, not logged SAGE-Fin outputs or comparative results, and estimate neither prevalence nor causal effect. A future privacy-preserving replay would pair independent labels with replayed gate outputs and a minimum row containing an anonymized case identifier, bucketed time, event and adapter type, candidate type, requested authority, coarse market and account state, dialogue state, witness status, actual outcome, human-or-legacy reason code, expected gate outcome, mismatch type, and bucketed latency. Public conformance, predecessor episodes, qualitative deployment feedback, and any future replay retain separate units and denominators.

\FloatBarrier

\subsection{Prospective Aggregate Replay}
\label{sec:replay}

The three predecessor episodes are qualitative retrospective reconstructions with second-analyst confirmation of predecessor non-interception (0/3). The separate later deployment has no linked aggregate outcomes, but a strongly positive formal post-deployment review organizationally independent of implementers and strongly positive end-user feedback provide distinct qualitative field corroboration of deployment value. Estimating prevalence or comparative interception requires a separately approved, privacy-preserving replay with independent labels and replayed gate outputs; causal effect requires a separate design. The following protocol is prospective and contributes no result in this paper. Table~\ref{tab:experiment-design} prespecifies six invariant families and the observation that would falsify each one; trace-reconstruction rate, predecessor-link completeness, hash mismatches, unexplained decisions, false blocks, and latency remain future measurements.

\begin{table}[H]
\caption{Prospective replay invariants and falsifiers. This protocol contributes no result in the current study.}
\label{tab:experiment-design}
\scriptsize
\begin{tabularx}{\columnwidth}{L{0.30\columnwidth}Y}
\toprule
Invariant & Mechanism check and falsifying result \\
\midrule
Receipt non-substitutability & Reject every nominal receipt type outside $\mathsf{Accept}(\mathsf{exec})$; falsified if a response commitment or progress receipt authorizes execution. \\
State-contingent invalidation & Recheck at consumption; falsified if execution proceeds after a bound drawdown, expiry, freshness, policy, regime, or incident condition fails. \\
Coverage-debt contraction & Missing or stale required witnesses reduce or block authority; falsified if a candidate with blocking debt receives adapter authority. \\
Dynamic budget and reservations & Open reservations and risk state constrain later caps; falsified if authority expands under an adverse-state transition or ignores an in-scope reservation. \\
Reconciliation feedback & Post-execution deviation updates future state; falsified if an overfill, slippage, or scope mismatch leaves registered feedback unchanged. \\
Policy deployment gate & Dialogue-derived policy requires deterministic IR and a deploy receipt; falsified if a candidate or progress receipt authorizes enforcement. \\
\bottomrule
\end{tabularx}
\end{table}

The target schema covers risk-control events, order or strategy-action traces, policy triggers, manual reviews, stale evidence, drawdown and exposure state, blocked or downgraded actions, rollback records, and timestamped state changes. Its anonymization rules, bucket definitions, exclusions, and baseline interfaces are maintained internally for a future approved replay; they are not a supplementary artifact in this submission. Each table row has a corresponding falsifier tied to Theorems~\ref{thm:nonsub}, \ref{thm:revoke}, and~\ref{thm:pol-sound} or Propositions~\ref{prop:cov-monotone} and~\ref{prop:contraction}.

\FloatBarrier

\section{Discussion}

\paragraph{Text and authority.}
Scenario~A instantiates the surface--authority separation in Remark~\ref{rem:surface-authority}: effectively unchanged SOP text is considered after dialogue history changes its candidate type and required witnesses. Reauthorization can therefore fail without any lexical change. The same distinction applies when an unchanged fee statement, risk note, order intent, or policy proposal is consumed under new freshness, portfolio, reservation, incident, policy, or user state. The predecessor episodes illustrate production failure transitions, not observed SAGE-Fin outcomes.

\paragraph{Evidence progress without authority escalation.}
SAGE-Fin separates progress in assembling evidence from permission to cause an effect. Evidence, budget, validation, fixture, and shadow receipts can record completed work or satisfy registered obligations, but adapter acceptance still depends on the adapter's nominal receipt type, exact-artifact binding, and current validity predicates. Proposition~\ref{prop:upstream-nonescalation} makes the handoff invariant explicit: upstream evidence activation or language validation may reduce coverage debt, yet cannot compose transitively into downstream permission. A budget record is not order authorization, and a validated policy is not deployment authorization. Missing witnesses remain visible as coverage debt; state-invalid receipts require re-issuance; and partial work can be retained without silently increasing authority.

\paragraph{Transfer across financial workflows.}
The authority coordinates are reusable, while witness registries and admissibility predicates remain institution- and product-specific. In brokerage, a market note becomes a recommendation only with suitability and communication witnesses. In derivatives, an option-chain snapshot supports a quote only within declared freshness and volatility bounds. In lending, a credit-line snapshot supports pre-approval only with current account state. In risk control, a policy proposal reaches enforcement only after deterministic IR, fixture and shadow checks, and a deployment receipt. These examples instantiate the same interface; they are not empirical claims of cross-domain performance. Field evidence in this paper is limited to the response adapter.

\section{Limitations}
\label{sec:limitations}

SAGE-Fin governs authority, not factual truth, investment return, causal alpha, or legal compliance. Its guarantees depend on its candidate compiler, witness registry, validators, and state feeds.

\textbf{Validator coverage.} Missing or unpromoted validators appear as coverage debt $D_t$ and contract authority under Proposition~\ref{prop:cov-monotone}. The runtime does not discover all relevant validators for a domain, and model-proposed validators remain candidate-only until promoted.

\textbf{Market behavior.} The execution gate enforces that any accepted instruction carries a valid \texttt{execAuth} bound to current state; it does not predict returns, establish causal alpha, or choose an optimal position. The optional HMM witness is evaluated only as a contractive execution-state input, not for regime-estimation accuracy or trading performance.

\textbf{Regulatory compliance.} SAGE-Fin provides auditable authority boundaries, receipt-gated execution, and deterministic policy enforcement. Whether those primitives satisfy a jurisdiction or supervisor depends on the institution's controls and legal analysis; the mechanism does not certify compliance.

\textbf{Compiler and language boundary.} Inputs that $\Gamma$ cannot parse fail closed, while incomplete compiled inputs surface as coverage debt. This can increase false rejection when compiler coverage is incomplete. The linguistic quality and factual support of realized claim content remain separate obligations; CBEA+LCV~\cite{tang2025recall} addresses upstream evidence activation and pre-realization language-commitment validation, whereas SAGE-Fin addresses downstream exact-artifact adapter consumption. An upstream pass can discharge a declared obligation but cannot supply adapter authority (Proposition~\ref{prop:upstream-nonescalation}).

\textbf{Runtime trust and replay.} The trace theorem is conditional on complete mediation and a trusted compiler, issuer, registries, clocks, feeds, and artifact bindings. Effects that bypass the adapters fall outside the guarantee. Candidate and exact-artifact hashes reject cross-artifact reuse, but the current prototype does not implement the atomic spent-receipt ledger needed to prevent repeated consumption of the same receipt for the same artifact.

\textbf{Principal empirical limitation: self-authored conformance.} The rules, mappings, and oracles were developed against the same catalog; 616/616 therefore tests declared semantics rather than independent safety accuracy, generalization, or extraction fidelity for production retrievers, dialogue trackers, and intent classifiers. The catalog includes three named response-gate fixtures; the issuance-only reference evaluator stratifies 526 cases as issuance-stage terminations and 90 as reaching full evaluation, 87 with explicit consumption-state objects and three testing diagnostic or receipt-type acceptance without one. Catalog-wide agreement is not external validation, deployed-code parity, reason-code equivalence, or an independently labeled safety estimate. The three distinct predecessor failures are qualitative and form the complete disclosure-approved set; a second analyst confirmed predecessor non-interception (0/3), but their SAGE-Fin actions remain retrospective policy implications. The separate later deployment establishes production use; a formal review organizationally independent of implementers and end-user feedback strongly corroborate practical usefulness and workflow fit, but do not yield an aggregate effectiveness estimate. Execution and policy field evidence remains prototype and reference-model evidence. Independent replay, false-block and latency measures, and human-validated end-to-end inputs remain necessary.

\textbf{Data release and reproducibility.} Authored fixtures, runner outputs, and the de-identified schema are maintained for release after the anonymity period, subject to repository and venue policy. No supplementary artifact accompanies this submission. Source production records, aggregate tables, exact thresholds, and venue-sensitive rules remain governed by the originating institution unless separately approved.

\section{Conclusion}

Financial market agents fail through hallucination, factual error, and bad forecasts; they also fail when context is mistaken for authority. A signal, quote, retrieved document, model analysis, or dialogue-derived instruction can be useful and even correct without being authorized to support a claim, recommendation, policy, or trade.

SAGE-Fin addresses this failure mode through a finance-specific compilation and authority-handoff contract. It compiles proposals into typed artifacts, binds them to promoted witnesses and validators, records missing coverage as debt, computes dynamic authority under current state, issues scoped receipts, gates exact-artifact consumption, and reconciles outcomes into future authority. Under complete mediation and the stated trust boundary, the finite-trace result shows that non-consumption progress cannot create a governed effect and that every effect requires a prior exact-artifact adapter receipt valid at consumption. The result is not a profitable trading agent; it is an auditable authority mechanism.

The reference evaluator matches its declared semantics across 616 authored cases, and the prototype matches its binary decision on all 616, including 3/3 named response-gate fixtures; the issuance-only reference evaluator stratifies them into 526 issuance-stage terminations and 90 full evaluations, while 22 tests cover selected enforcement paths. Because the mapping and evaluator share the authored catalog, this is implementation parity rather than independent accuracy. Three distinct predecessor failures, with second-analyst confirmation that predecessor interception was 0/3, ground the failure taxonomy; separate later response-gate deployment, a strongly positive formal review independent of implementers, and strongly positive end-user feedback establish bounded use and provide implementation-independent qualitative corroboration of practical usefulness and workflow fit. These artifacts support mechanism conformance, independently confirmed failure grounding, and qualitative deployment value---not a comparative-effectiveness estimate or causal improvement.

\FloatBarrier
\begingroup
\interlinepenalty=10000
\bibliographystyle{ACM-Reference-Format}
\bibliography{references}
\endgroup

\end{document}